\documentclass{article}

\PassOptionsToPackage{numbers,compress}{natbib}
\usepackage[preprint]{neurips_2026}

\usepackage[utf8]{inputenc}
\usepackage[T1]{fontenc}
\usepackage{hyperref}
\usepackage{url}
\usepackage{booktabs}
\usepackage{amsfonts}
\usepackage{nicefrac}
\usepackage{microtype}
\usepackage{xcolor}
\usepackage{subcaption}
\usepackage{multirow}
\usepackage{amsmath}
\usepackage{amssymb}
\usepackage{mathtools}
\usepackage{amsthm}
\usepackage{makecell}
\usepackage{tabularx}
\usepackage{bm}
\usepackage{algorithmic}
\usepackage{algorithm}
\newcommand{\argmax}{\mathop{\rm argmax}\limits}

\usepackage[capitalize,noabbrev]{cleveref}
\usepackage{graphicx}
\theoremstyle{plain}
\newtheorem{theorem}{Theorem}[section]
\newtheorem{proposition}[theorem]{Proposition}
\newtheorem{lemma}[theorem]{Lemma}

\theoremstyle{definition}

\theoremstyle{remark}

\title{An Information Theoretic Framework for Graph Novelty Generation via Latent Mixture Modeling}

\author{
  Itsuki Nakagawa$^{1}$ \quad Kenji Yamanishi$^{1}$ \\
  $^{1}$Graduate School of Information Science and Technology, The University of Tokyo \\
  \texttt{\{itsukin0310, yamanishi\}@g.ecc.u-tokyo.ac.jp}
}

\begin{document}

\maketitle

\begin{abstract}
We propose an information-theoretic framework for graph novelty generation, which aims to generate data that are distinct from existing patterns while preserving global structural consistency.
Our approach embeds data into a latent space, models the latent distribution using finite mixture models, and generates novel samples by imposing explicit novelty and reliability conditions formulated in terms of description length.
Specifically, novelty is enforced by requiring generated samples to be poorly explained by all existing mixture components, while reliability constrains their impact on the overall mixture structure under the Minimum Description Length (MDL) principle.
We provide a theoretical analysis showing that, with appropriate threshold choices, the probabilities of misclassifying non-novel or unreliable samples converge to zero with explicit rates.
Experiments on synthetic and benchmark graph datasets demonstrate that the proposed method enables principled novelty generation with quantifiable risk.
\end{abstract}

\section{Introduction} \label{intro}
\subsection{Motivation}

With the rapid progress of contemporary AI technologies, approaches such as data augmentation, extrapolation, and out-of-distribution (OOD) learning--which aim to handle or generate data not present in the training set--have been extensively studied. In this paper, we focus on a more fundamental yet less explored problem, which we refer to as {\em novelty generation}.

Novelty generation aims to create data that lie beyond the original data distribution while maintaining sufficient reliability. Such a capability is essential for applications that require genuine creativity~(\cite{vc19}) rather than distribution-preserving variation, such as the formation of new communities or the design of new materials. Unlike data augmentation or extrapolation, which preserve or extend existing patterns, novelty generation seeks to introduce data that do not belong to any known patterns. Despite its importance, novelty generation lacks a rigorous mathematical formulation, making it difficult to design principled algorithms with theoretical guarantees.

It is instructive to distinguish novelty generation from related concepts. Data augmentation generates samples along the original data distribution to compensate for data scarcity, and has been widely studied, including in the context of graph data~(see survey papers,
e.g. \cite{zxwzx20,zlnwjs21,yhdx22,dxtl22,azxa23,zjlwlgsj23,ywczs25}). Extrapolation focuses on extending known patterns beyond the observed data, treating unseen samples as natural continuations of existing structures~(e.g. \cite{hnclhx19,bsjs21,lydlxzhw22}). Novelty generation, in contrast, aims to alter the original distribution itself in a controlled manner. Similarly, while OOD learning addresses unseen data, it primarily focuses on detection and recognition~(\cite{hg16}), whereas novelty generation emphasizes the creative synthesis of new data instances.

The purpose of this paper is threefold. First, we propose an information-theoretic framework for graph novelty generation based on latent mixture modeling. Second, we formulate novelty and reliability as description-length-based conditions and introduce MDL-guided sampling algorithms with theoretical guarantees. Third, we empirically demonstrate the effectiveness of the proposed framework using both synthetic and real-world datasets.

\subsection{Significance and Novelty}
Most existing generative approaches share a common methodology; original data are embedded into a latent space, latent representations are manipulated, and new data instances are generated by decoding the manipulated latent points.
The proposed framework follows this general paradigm of latent space manipulation but introduces the following novel and significant ideas and analysis that enable principled novelty generation.

A) {\em Latent Space Modeling with Finite Mixture Models}.
We model the probability distribution of latent representations using a finite mixture model (FMM), such as a von Mises-Fisher mixture. Under this formulation, the original data distribution is represented by multiple mixture components, and novelty is characterized by latent samples that cannot be explained by any existing component.

B) {\em MDL-Guided Sampling}.
Generated novel samples must satisfy two conflicting requirements: they should be sufficiently distant from all existing mixture components ({\em the novelty condition}), while not significantly disturbing the global structure of the mixture model ({\em the reliability condition}). The central contribution of this work is to formulate these conditions in terms of description length under the Minimum Description Length (MDL) principle.
Specifically, we propose an MDL-guided sampling scheme that enforces novelty through a large local increase in description length, while ensuring reliability by bounding the global description-length change of the mixture model. A generated sample is accepted as novelty only if it satisfies both criteria, which are treated as distinct acceptance conditions rather than absolute guarantees.

C) {\em False Novelty and Reliability Probability Analysis}. Our theoretical analysis focuses on controlling  misclassification probabilities. We show that, under appropriate choices of novelty and reliability thresholds, (i) the probability that a sample satisfying the novelty condition is in fact generated from an existing mixture component, and (ii) the probability that a sample satisfying the reliability condition is generated from a distribution that significantly deviates from the original mixture, both vanish to zero with explicit convergence rates.

D) {\em Empirical Demonstration.} We employ synthetic and real data to show that novelty and reliability are more effectively controlled by varying threshold parameters than any competing methods. We thereby confirm that our method enables principled novelty generation with quantifiable risk.

\subsection{Related Work}

Latent space manipulation has become a fundamental technique in generative AI. Data augmentation has been realized through latent interpolation and sampling in GAN-based models~(\cite{liu18,tvsl24,cy21,kzvw20}), while extrapolation has been studied using GANs and encoder-decoder architectures that modify latent distributions or attributes~(\cite{cmkn21,fxzz22,pdkb22}). Latent space manipulation has also played a central role in OOD learning~(\cite{llls18,rrke22}).

In the context of novelty-related generation, several studies have explored knowledge extrapolation or creative generation using regularized encoder-decoder frameworks,  GAN-based methods~(\cite{fxzz22,cmkn21}), diffusion methods (\cite{ho20,dn21,vkswcf23,hsdfl22}). While these approaches demonstrate empirical success, they do not provide a principled formulation of novelty generation with explicit guarantees.

Latent space modeling using finite mixture models has been investigated in various generative and representation learning settings, including deep clustering~(\cite{czccc17,yclf19}), graph augmentation~(\cite{yanjin}),
semi-supervised learning~(\cite{hzhq20}), and self-supervised learning with von Mises-Fisher mixtures~(\cite{vmfmixture}). These works primarily focus on modeling or augmenting existing data distributions, rather than generating
novelty beyond all mixture components.

Graph community augmentation has been argued in \cite{icdm24} under the reliability and novelty conditions.
It differs from our work in that they used distance-based formulations of the conditions rather than description-length based ones; they considered Gaussian mixtures rather than vMF mixtures, and they didn't give any theoretical analysis of misclassification.

The MDL principle has been extensively applied to statistical modeling and model selection~(\cite{rissanen78,gwd,mdlbook}), including the selection of latent variables via criteria such as DNML~(\cite{dnml}). While MDL has recently been applied to data augmentation~(\cite{yanjin}), no existing work has formulated novelty generation as an MDL-guided sampling problem in latent mixture models.

To our knowledge, this paper is the first to propose a framework for graph novelty generation that combines finite mixture modeling of latent space with MDL-guided sampling, providing both theoretical guarantees and empirical validation on the controllability of novelty and reliability.
To ensure reproducibility, implementation details are provided in Appendix F.

\section{Overall Flow of Novelty Generation }

\begin{figure}[t!]
\centering
 \vskip -0.2in
    \centering
    \includegraphics[width=7.2cm]
    {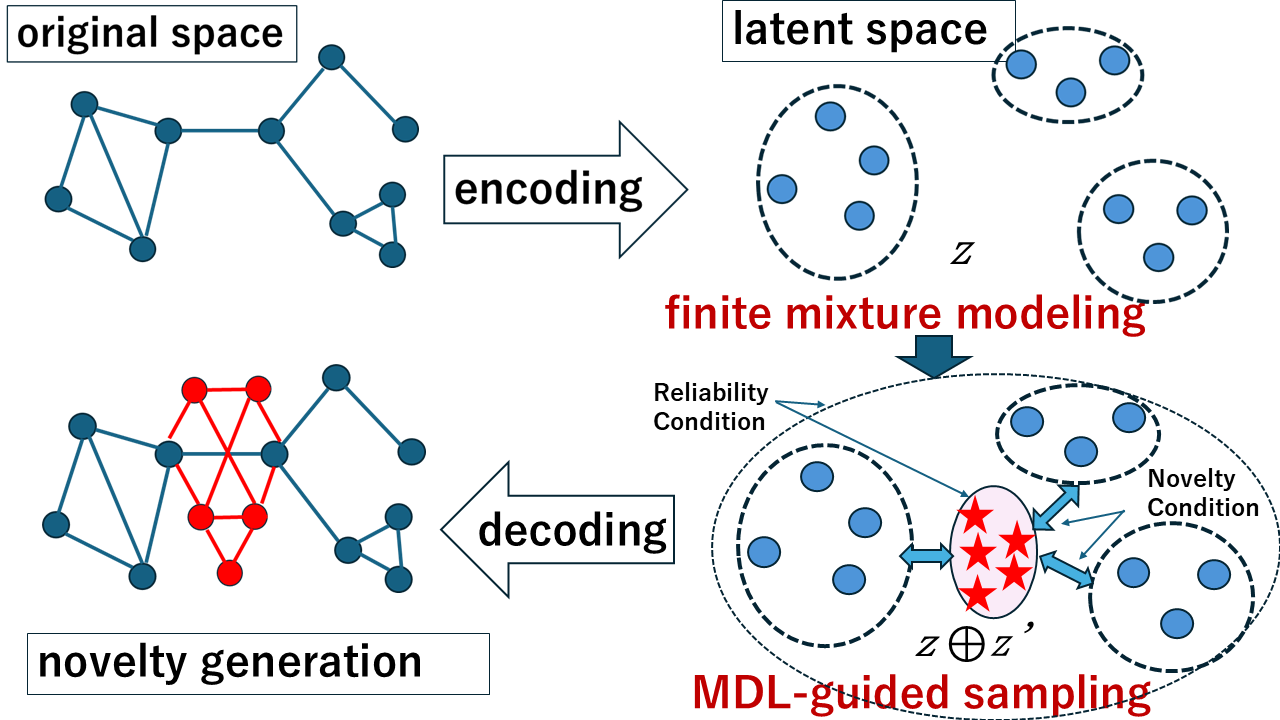}
   \caption{The flow of our framework. Red stars show sampled points decoded to a novel graph.
   }
   \label{flow}
   \vskip -0.1in
\end{figure}

While the proposed framework is formulated at a general level and is applicable to a broad class of generative models,
we focus in this paper on graph-structured data as a concrete
 instantiation.
In particular, this section describes the overall flow of novelty generation in the context of community-level graph generation,
which serves as a testbed for our framework.
The overall flow consists of the following steps: 1) encoding, 2) latent mixture modeling, 3) MDL-guided sampling, and 4) decoding (see Figure \ref{flow}).

1) {\em Encoding step}:
First, we give an encoding and decoding framework.
For a binary graph $G=(V,E )$\ ($V$: a set of vertices, $E$: a set of edges),  let ${\mathcal A}=(y_{ij})$ be its adjacency matrix for $G$.
Let
${\mathcal Z}$ be the latent space
and a map $\psi: V \rightarrow {\mathcal Z}$ be an embedding.
The $i$-th node in $V$ is encoded to a latent state vector $z_{i}\in {\mathcal Z}$.
$\psi$ may be modeled using a graph auto encoder~(GAE).
We define the reconstruction  loss by
\begin{align}\label{deconstruction}
&{\mathcal L}_{{\rm rec}}=\sum _{(i,j): y_{ij}=1}\left(-\log \sigma \left(\tau z_{i}^{\top}z_{j}\right)\right)+\sum _{(i,j): y_{ij}=0}\left(-\log \left(1-\sigma \left({\tau z_{i}^{\top}z_{j}}\right)\right)\right),
\end{align}
where $\sigma$ is a sigmoid function: $\sigma (x)=1/(1+\exp (-x))$ and $\tau>0$ is a temperature parameter to ensure sufficient dynamic range in the edge probabilities.
The embedding parameters are optimized using negative sampling and gradient descent to minimize the reconstruction loss.

2) {\em Latent mixture modeling step}: Once a graph is encoded into the latent space, we model it with a {\em finite mixture model} (FMM), where each cluster in FMM corresponds to a graph community.

3) {\em MDL-guided sampling step}:
We make sampling of latent vectors so that the sampled points form a group that does not change the overall FMM significantly ({\em reliability condition}), while it is far from all existing mixture components in FMM ({\em novelty condition}).
These properties are specified in terms of description length, under the  MDL principle.
MDL is used as an information-theoretic criterion to define acceptance conditions for novelty and reliability, not as an optimization objective.

4) {\em Decoding step}: We reconstruct the existing graph and generate novel structures via deterministic decoding with a density-preserving threshold.

\section{Latent Mixture Modeling}

This section yields the details of vMF-mixtures.
We decompose the latent variable into
$z=(r, \phi )$ where $r \in {\mathbb R}$ shows the radius norm, and  $\phi \in {\mathbb S}^{d-1}$ shows the angle.
Let $\mu$ with $||\mu ||=1$ and $\lambda \in (0, \infty)$ be
parameters.
A probability density function of $\phi$,
denoted as
$p(\phi; \mu, \lambda)$,
is modeled by {\em von Mises-Fisher~(vMF) distribution}:
\[p(\phi; \mu, \lambda)=\frac{\exp (\lambda \mu^{\top}{\phi})}{C(\lambda)},\ \ \ C(\lambda )=\frac{(2\pi)^{d/2}I_{d/2-1}(\lambda )}{\lambda ^{d/2-1}},\]
where $I_{d/2-1}$ is the modified Bessel function of the first kind of order  $d/2-1$.
$\mu$ and $\lambda$ can be thought of as a mean vector and a variance in directional statistics.
A probability density function of $r$, denoted as $p(r; \xi)$ is modeled by a 1-dimensional Gaussian distribution with parameter $\xi$(mean and variance).
We let $p(z; \xi, \mu, \lambda)=p(r;\xi)p(\phi ; \mu, \lambda)$.

Introducing a latent variable $w$ indicating a cluster index, for a positive integer $k$~(mixture size), we consider a {\em vMF-mixture}~(\cite{vmfmixture}) over the latent space, for which the complete variable model is given as follows:
\begin{eqnarray*}
p({z},w;\theta )&=&\prod _{j=1}^{k}(\pi _{j}p(r; \xi_{j})p(\phi ; \mu_{j}, \lambda _{j}))^{\delta (w=j)}\
 \end{eqnarray*}
 where $\pi_{j}\geq 0$, $\sum _{j=1}^{k}\pi _{j}=1$, and $\delta$ is the delta function. We set $\theta =(\{\pi _{j}, \xi_{j}, \mu _{j}, \lambda _{j}\})$.

For a given independent data sequence ${{\bm z}}={z_{1}},\dots ,{z_{n}}$,
the EM~(Expectation and Maximization) algorithm outputs the following statistics and the maximum likelihood estimates (MLE) of parameters iteratively~(\cite{vmfmixture}):
\begin{align*}
& p(j|{z_{i}}; \xi_{j}, \mu _{j},\lambda _{j})=\frac{\pi _{j}p({z_{i}}; \xi_{j}, \mu _{j}, \lambda _{j})}{\sum _{\ell =1}^{k}\pi _{\ell}p({z_{i}}; \xi_{\ell}.\mu _{\ell}, \lambda _{\ell})},\ \ w_{i}=\argmax _{j}p(j|{z_{i}}; \xi_{j}, \mu _{j}, \lambda _{j}),\\
& \pi _{j}=\frac{1}{n}\sum _{i=1}^{n}p(j|{z_{i}}; \xi_{j},\mu _{j}, \lambda_{j}),\ \
\gamma_{j}= \sum _{i=1}^{n}{\phi_{i}}p(j|{z_{i}}; \xi_{j},\mu _{j}, \lambda _{j}),\\
& {\mu}_{j}=\frac{\gamma_{j}}{||\gamma_{j}||},\ \ \
{R}_{j}=\frac{||\gamma_{j}||}{\sum ^{n}_{i=1}p(j|{z_{i}}; \xi_{j},\mu _{j},\lambda _{j})},\ \ \
{\lambda}_{j}=\frac{R_{j}(d-R_{j}^{2})}{1-R_{j}^{2}}.
\end{align*}
The Gaussian parameter $\xi$ is also estimated using the EM algorithm in a standard way.

\section{MDL-Guided Sampling}

\subsection{Preliminary 1: NML Code-length}

In the literature of the {\em MDL principle}~(\cite{optimal,gwd,mdlbook}),
the information included in a data sequence is measured in terms of its {\em description length}, which is formally defined as
the {\em normalized maximum likelihood (NML) code-length}.
This is  defined as the logarithmic loss for the data relative to the NML distribution.
The description length is used as a criterion to define acceptance conditions for novelty and reliability defined later.

In general, for a class of parametric probability distributions ${\mathcal C}=\{p(z; \theta )\}$, the NML code-length ${\mathcal L}_{_{\rm NML}}({\bm z})$ for a given data sequence ${\bm z}=z_{1},\dots ,z_{n}$ is calculated as
\begin{equation}\label{nml}
{\mathcal L}_{_{\rm NML}}({\bm z})\buildrel \rm def \over =-\log \max_{\theta}p({\bm z}; \theta)+\log C_{n},\ \ \
C_{n}\buildrel \rm def \over =\int \max_{\theta}p({\bm z}; \theta)d{\bm z},
\end{equation}
where we call $\log C_{n}$ the {\em parametric complexity} of ${\mathcal C}$. It depends on sample size and the class only.

We derive a formula of parametric complexity for vMF distribution.

\begin{proposition}\label{prop1}
The parametric complexity
for the vMF distribution
satisfies
\begin{align}\label{vmfnml}
&\log C_{n}
\leq
\frac{d}{2}\log \frac{n}{2\pi}+
\log \frac{\pi ^{\frac{d+1}{2}}(d+1)\Gamma (\frac{d-1}{2}) }{\sqrt{d}(d-1)^{\frac{d-1}{2}}\Gamma (\frac{d}{2})^2},
\end{align}
The full proof is given in Appendix \ref{proofprop}.
\end{proposition}
We also give a formula of the NML code-length relative to the GMM in Appendix \ref{gaussiannml}.
\subsection{Preliminary 2: DNML Code-length for FMMs}

Once the mixture model
is obtained, we may calculate the total code-length for the data sequence relative to the complete variable model.
Let ${\mathcal W}=\{1,\dots ,k\}$ be a set of cluster indexes, which is a latent variable.
For an observed independent sequence ${\bm z}=z_{1},\dots , z_{n}$, let the corresponding cluster sequence be ${\bm w}=w_{1},\dots ,w_{n}$,
the {\em decomposed normalized maximum likelihood (DNML) code-length}~(\cite{dnml}) for ${\bm z}$ and ${\bm w}$ is given by the sum of the NML code-length for ${\bm z}$ given ${\bm w}$ and that for ${\bm w}$.
That is,
\begin{equation}\label{dnml}
{\mathcal L}_{_{\rm DNML}}({\bm z},{\bm w};k)
=\sum ^{k}_{j=1}{\mathcal L}_{_{\rm NML}}({\bm z}_{j})+{\mathcal L}_{_{\rm NML}}({\bm w}),
\end{equation}
where ${\bm z}_{j}$ is the subsequence of ${\bm z}$ such that $w=j$.
${\mathcal L}_{_{\rm NML}}({\bm w})$ is the NML code-length for ${\bm w}$ with respect to the multinoulli distribution; $\{p(w=i)=\pi _{i}, \sum_{i=1}^{k}\pi _{i}=1,\ \forall \  i,  \pi _{i}\geq 0\}.$
Letting $H$ be the entropy function, this is calculated as
\begin{eqnarray*}
& &{\mathcal L}_{_{\rm NML}}({\bm w})=nH\left(\frac{n_{1}}{n}\cdots \frac{n_k}{n}\right)+\log C_{n}(k),
\end{eqnarray*}
where $n_{j}=|{\bm z}_{j}|\ (j=1, \dots ,k),$ and $C_{n}(k)$ is a parametric complexity for the multinoulli distribution.
According to \cite{km07}, $C_{n}(k)$ can be computed in time $O(n+k)$ using the following relation:
$C_{n}(k+2)=C_{n}(k+1)+({n}/{k})C_{n}(k)$.
The optimal $k$ can be determined by minimizing (\ref{dnml}) or using ground-truth labels. We use this $k$ for FMM hereafter.

\subsection{Conditions for MDL-guided sampling}

Let ${\bm z}$ be normal examples in the latent space, modeled by an FMM.
Let ${\bm z}'$ be generated data in the latent space.
In order to ensure that ${\bm z}'$ is novel with some reliability, we require that ${\bm z}'$ is distinct from any component in the FMM
({\em novelty condition}),
while we require that ${\bm z}$ reasonably fit within the overall FMM in distribution ({\em reliability condition}).
Below we formalize these conditions in terms of description lengths (see Figure \ref{flow}). We do not optimize description length directly.
Instead, it is used as an information-theoretic criterion to define acceptance conditions for novelty and reliability.

Let ${\bm z}_{i}$ be a subsequence of ${\bm z}$  whose datum is categorized into the $i$-th component of the FMM via e.g. EM algorithm.
The novelty condition is formulated in terms of the NML code-length for ${\bm z}'$ relative to a non-mixture model.\\
\ \ \\
{\bf Novelty condition}: For a parameter $\epsilon _{1}\in {\mathbb R}$,
the generated data set ${\bm z}'$ should satisfy:
\begin{align}\label{noveltycond}
& \forall i\in \{1, \dots , k\}, \ \ {\mathcal L}_{_{\rm NML}}({\bm z}_{i}\oplus {\bm  z}')-
\{{\mathcal L}_{_{\rm NML}}({\bm z}_{i} )+{\mathcal L}_{_{\rm NML}}( {\bm z}' )\}> |{\bm z}_{i}\oplus {\bm z}'|\epsilon _{1},
\end{align}
where ${\bm z}_{i}\oplus {\bm  z}'$ means a sequence obtained by concatenating ${\bm z}'$ and ${\bm z}_{i}$.
We call the left-hand side of (\ref{noveltycond})  {\em novelty score}.
Eq.~(\ref{noveltycond}) shows that for any $i$, the NML code-length required for encoding ${\bm z}_{i}$ and $ {\bm z}'$ with distinct models separately  can significantly reduce that required for  encoding ${\bm z}_{i}\oplus {\bm z}'$ with a single model simultaneously.
Intuitively, this implies that for each $i$, ${\bm z}'$ is significantly distinct from ${\bm z}_{i}$ in distribution.

Meanwhile,
letting ${\bm w}$ denote the latent variable indicating the cluster corresponding to ${\bm z}$, the reliability condition is specified in terms of the DNML code-length for ${\bm z}\oplus {\bm z}'$ relative to the FMM.
\\
\ \ \\
{\bf Reliability condition:}
For a parameter $\epsilon_{2}\in {\mathbb R}$, the generated data set ${\bm z}'$ should satisfy
\begin{align}\label{reliabilitycond}
&{\mathcal L}_{_{\rm DNML}}({\bm z}\oplus {\bm z}', {\bm w}\oplus{\bm  w}')-
\{{\mathcal L}_{_{\rm DNML}}({\bm z}, {\bm w})+{\mathcal L}_{_{\rm DNML}}( {\bm z}', {\bm w}')\}\leq (|{\bm z}\oplus {\bm z}'|+|{\bm  w}\oplus {\bm w}'|)\epsilon _{2}.
\end{align}
We call the left-hand side of (\ref{reliabilitycond})  {\em reliability  score}.
Note that ${\bm w}'$ is a single component different from any of ${\bm w}$.
Eq. (\ref{reliabilitycond}) shows that the DNML code-length required for separately encoding $({\bm z},{\bm w})$ and $({\bm z}', {\bm w}')$ with distinct models does not significantly reduce that required for simultaneously encoding $({\bm z}\oplus {\bm z}', {\bm w}\oplus {\bm w}')$ with a single model.
Intuitively, this implies that the distribution of $({\bm z}, {\bm w})$ does not change so much even if $({\bm z}', {\bm w}')$ is incorporated.

{\em MDL-guided sampling} refers to sampling of data in the latent space under the novelty and reliability conditions.
This is basically a rejection sampling where ${\bm z}'$ is sampled according to a proposal distribution, then accept ${\bm z}'$ if our conditions are satisfied, otherwise reject it.

We design the proposal distribution $q_{g}$ as follows:
As for angle component of vMF mixtures with clusters with $\{(\lambda _{j},\mu _{j})\} _{j=1,\dots ,k}$,
we randomly choose a form of  $\tilde{\mu}=\sum _{j}\alpha _{j}\mu _{j}+\epsilon \ (\epsilon \sim {\mathcal N}(0, (\sigma)^{2}{ I})$ from the convex hull of  means of mixture components, and set
$\bar{\mu}=\tilde{\mu}/||\tilde{\mu}||$ and $\bar{\lambda}=\sum _{j}\alpha _{j}\lambda _{j}$.
The proposal distribution  is the vMF with $(\bar{\lambda}, \bar{\mu})$.
The radius component is given in the same manner.

\vskip -0.02in
\begin{algorithm}[!t]

  \caption{Novelty Generation Algorithm (NVG)}
  \label{alg:nvg}

  \begin{algorithmic}\label{nvg-alg}
    \STATE {\bfseries Input:} Graph: $G=(V,E)$,  parameters: $\epsilon_{1}, \epsilon _{2}$, $n'$: number of sample points, $m:$ sample size
    \STATE {\bfseries 1:} Encode $G$ into the latent space ${\mathcal Z}$ to produce ${\bm z}$ as a latent representation of $G$.
    \STATE {\bfseries 2:} Estimate  a finite mixture model of  ${\bm z}$ to produce estimates of parameters, a sequence of latent variables ${\bm w}$ and
${\bm z}_{i}\ (i=1,\dots ,k)$ where ${\bm z}_{i}$ is data
belonging to the $i$th component ($w=i$).
\STATE Initialize sample$=0$
   \REPEAT

    \STATE{\bfseries 3:} Generate ${\bm z}'$ by random sampling of $n'$ points according to the proposal distribution $q_{g}({\bm z}')$,
     and let $w'$ be a new cluster index vector corresponding to ${\bm z}'.$
\IF{The novelty condition (\ref{noveltycond}) and reliability condition (\ref{reliabilitycond}) are fulfilled}
\STATE {\bfseries 4:} Decode ${\bm z}\oplus {\bm z}'$ to produce a novel graph.
    \STATE sample= sample$+1$
    \ENDIF
    \UNTIL{ sample $=m$}.
\STATE {\bfseries Output:} $m$ graphs decoded from $\{{\bm z}\oplus {\bm z}'\}$.
  \end{algorithmic}
\end{algorithm}

We propose a {\em  novelty generation algorithm  based on MDL-guided sampling} (NVG) as in Algorithm \ref{nvg-alg}.
The time complexity for the encoding step is $O(I_{1}|E|)$, while the estimation step is $O(I_{2}(n+n')kf(d))$ where
$f(d)=d$ for vMF mixtures and $f(d)=d^{2}$ for GMM. $I_{1}$ and $I_{2}$ are the number of iterations for GAE and FMM estimation, respectively. The time complexity for  parametric complexity for NML or DNML code-length is $O(n+n'+k)$.
The time complexity for MDL-guided sampling is $O(m/p)$ where $p$ is the probability that a candidate sample satisfies the acceptance conditions.
This $p$ scales asymptotically as $O(\epsilon_2^{d/2})$.
There is a near-linear tradeoff between $\epsilon _1$ and $\epsilon_2$; thus, increasing $\epsilon_1$ leads to a decrease in $p$ on the same order.
Although this rate suffers from the curse of dimensionality, in practice, we mitigate this by using a low latent dimension ($d\sim 8$) (see Sec.\ref{sec-exp}) in our experiments.
Higher dimensional cases are left for future study.

\section{Theoretical Analysis}\label{theory}

Generated samples must satisfy novelty and reliability, which are potentially conflicting criteria.
We analyze the misclassification probabilities associated with each condition and show that, under suitable threshold choices, both probabilities converge to zero with explicit rates.
We analyze novelty and reliability separately, without claiming their simultaneous satisfaction.

Let the components of the FMM be $\{{\bm z}_{i}\}_{i=1,\dots ,k}$ and the distribution for ${\bm z}_{i}$ be $p_{\theta_{i}}(z)$\ $(i=1,\dots ,k)$.
We define ${\mathcal E}_{1}$ as the event that for some $i$,
${\bm z}_{i}$ and ${\bm z}'$ satisfying the novelty condition (\ref{noveltycond}) are identically distributed according to $p_{\theta _{i}}(z)$.
Theorem \ref{thm-novelty} gives an upper bound on the probability that  ${\mathcal E}_{1}$ occurs, which means the probability of false novelty  for  MDL-guided sampling.

\begin{theorem}\label{thm-novelty}{\rm (Probability of False Novelty)}
Let $n_{i}$ be the size of ${\bm z}_i\ (i=1,\dots , k)$ and $n'$ be the size of ${\bm z}'$. Let $C_{n_{i}+n'
}$ be the parametric complexity as in (\ref{nml}).
The probability that ${\mathcal E}_{1}$ occurs
 is upper-bounded as follows:
\vspace{-0.1in}
\begin{equation}
{\rm Prob}[{\mathcal E}_{1}]\leq
\sum _{i=1}^{k}\exp \left( -(n_{i}+n')\epsilon _{1} +\log C_{n_{i}+n'}\right).
\end{equation}
\end{theorem}
{\em Proof sketch.} Let the probability density that ${\bm z}_{i}$ and ${\bm z}'$ occur according to the same distribution regardless the novelty condition holds be $p_{i}({\bm z}_{i}, {\bm z}')$.
We can show that this is upper-bounded as:
\begin{align*}
p_{i}({\bm z}_{i}, {\bm z}')
\leq e^{-{\mathcal L}_{_{\rm NML}}({\bm z}_{i})}e^{-{\mathcal L}_{_{\rm NML}}({\bm z}')}
e^{-(n_{i}+n')\epsilon _{1}+\log C_{n_{i}+n'}},
\end{align*}
Therefore, letting the integrals be taken over the domain satisfying (\ref{noveltycond}), we have
\begin{align*}
& {\rm Prob}[{\mathcal E}_{1}]\leq \sum _{i=1}^{k}
\iint d{\bm z} d{\bm z}'
p_{i}({\bm z}_{i}, {\bm z}')
\leq \sum _{i=1}^{k}e^{-(n_{i}+n')\epsilon _{1}+\log C_{n_{i}+n'}}.
\end{align*}
 To derive the last inequality, we used the Kraft inequality $\int d{\bm z}e^{-{\mathcal L}_{_{\rm NML}}({\bm z})}\leq 1$ for the NML code-length. This holds even if the parametric complexity (\ref{nml}) is replaced with its upper-bound.
If the penalty-term is less than the parametric complexity, the Kraft inequality does not hold and the bound is not valid.  The usage of the NML code-length is essential. The full proof is in Appendix \ref{proofth1}. $\Box$

Theorem \ref{thm-novelty} implies that the probability that
a sample satisfying the novelty condition is in fact generated from an existing mixture component approaches to zero exponentially as all $(n_{i}+n')$s increase, provided that $\epsilon _{1}>(\log C_{n_{i}+n'})/(n_{i}+n')$ for all $i$.
The rate of convergence depends on the parametric complexity
$C_{n_{i}+n'}$.
This guarantees, with high confidence, that the data generated under the novelty condition can be well discriminated from all of the components of the mixture.

Next let us define ${\mathcal E}_{2}$ as the event that
the reliability condition (\ref{reliabilitycond}) is satisfied for
 $({\bm z}, {\bm w})$ and $({\bm z}', {\bm w}')$ even though they are generated from substantially different distributions:
That is, $(z,w)\in ({\bm z},{\bm w})\sim p_{\theta }(z,w)=p_{\theta_{1}}(z|w)p_{\theta _{2}}(w)$ and
$(z',w')\in ({\bm z}',{\bm w}')\sim p_{\theta' }(z',w')=p_{\theta_{1}'}(z'|w')p_{\theta_{2}'}(w')$ where $\theta =(\theta _{1}, \theta _{2})\neq (\theta _{1}', \theta _{2}')=\theta'$.
Let $|{\bm z}| =|{\bm w}|=n$ and $|{\bm z}'|=|{\bm w}'|=n'$.
We relate this probability to
the distance  between a model that incorporates ${\bm z}'$ into ${\bm z}$ and treats as a single distribution, and an alternative model that assumes ${\bm z}$ and ${\bm z}'$ are generated from different distributions.
Theorem \ref{thm-reliability} evaluates the probability that ${\mathcal E}_{2}$ occurs for  MDL-guided sampling (the false reliability probability).

\begin{theorem}\label{thm-reliability}{\rm (Probability of False Reliability)}
The probability that ${\mathcal E}_{2}$ occurs
is upper-bounded as:
\begin{eqnarray}
{\rm Prob}[{\mathcal E}_{2}]&\leq &
\exp\Big[-2(n+n'){\rm d}(p_{_{\rm DNML}}, p_{\theta }\otimes p_{\theta'})+
\frac{1}{2}\log \bar{C}_{n}C_{n'}+(n+n')\epsilon_{2}\Big],
\end{eqnarray}
where $\bar{C}_{n}=\sup _{{\bm w}}C_{n}$
is an upper bound on the parametric complexity for the FMM (the supremum is achieved for $n_{i}=n/k$ for each $i$),
and ${\rm d}$ is the Bhattcharyya distance between the concatenated density $p_{\theta }\otimes p_{\theta'}$ and the DNML probability density defined as $p_{_{\rm DNML}}=e^{-{\mathcal L}_{_{\rm DNML}}({\bm z}\oplus {\bm z}', {\bm w}\oplus {\bm w}')}$.
The Bhattacharyya distance
is defined as ${\rm d}(p_{1},p_{2})=-(1/m)\log \int( p_{1}({\bm y})p_{2}({\bm y}))^{\frac{1}{2}}d{\bm y}$ where $m=2(n+n')$ is the length of the sequence ${\bm y}$.

\end{theorem}
{\em Proof sketch.}
The probability that ${\mathcal E}_{2}$
occurs
is given by
\begin{equation}\label{11}
{\rm Prob}[{\mathcal E}_{2}]
 = \iiiint
 d{\bm z}d{\bm z}'d{\bm w}d{\bm w}'
p_{\theta}({\bm z}, {\bm w})p_{\theta'}({\bm z}', {\bm w}'),
\end{equation}
where the integrals are taken so that  (\ref{reliabilitycond}) is satisfied.
Meanwhile, the reliability condition leads to
\begin{align}\label{12}
&1\leq \left(e^{-{\mathcal L}_{\rm DNML}({\bm z}\oplus {\bm z}', {\bm w}\oplus {\bm w}')}/p_{\theta}({\bm z}, {\bm w})
{p_{\theta'}({\bm z}', {\bm w}')}\right)^{\frac{1}{2}}e^{ \frac{1}{2}\log \bar{C}_{n}C_{n'}+(n+n')\epsilon _{2}}. \nonumber
\end{align}
Multiplying the lefthand side of this by that of (\ref{11}) gives
\begin{align*}
& {\rm Prob}[{\mathcal E}_{2}] \leq e^{-2(n+n'){\rm d}(p_{_{\rm DNML}}, p_{\theta}\otimes p_{\theta'})+\frac{1}{2}\log \bar{C}_{n}C_{n'}+(n+n')\epsilon _{2}}.
\end{align*}
We used the property that $e^{-{\mathcal L}_{_{\rm DNML}}({\bm z}\oplus {\bm z}', {\bm w}\oplus {\bm w}')}$ forms a probability density over ${\mathcal Z}^{n+n'}\times {\mathcal W}^{n+n'}$, because of the property of DNML code-length. The bound does not hold for any other loss functions with less penalty terms than the DNML. The full proof is given in Appendix \ref{proofth2}. \ \ \ \ \ \ \ \ \ \ \ \ \ \ $\Box$

The DNML probability density function
$p_{_{\rm DNML}}$ is the mixture model obtained by incorporating the sample ${\bm z}' $ into  the original mixture, while $p_{\theta}\otimes p_{\theta'}$ is the true model that ${\bm z}$ and ${\bm z}'$ are independently generated according to different distributions.
Theorem \ref{thm-reliability} implies that the probability that
a sample satisfying the reliability condition is generated according to the  true distribution $p_{\theta}\otimes p_{\theta'}$ that significantly deviates from the incorporated mixture, approaches to zero exponentially as  $(n+n')$ increases, provided that
$d(p_{_{\rm DNML}}, p_{\theta}\otimes p_{\theta'})$
exceeds $(1/4)(\log \bar{C}_{n}C_{n'} )/(n+n')+\epsilon _{2}/2$.
This guarantees with high probability that the data generated under the reliability condition cannot be generated from the distribution that is significantly deviated from the incorporated mixture.

\section{Experimental Results}\label{sec-exp}

\subsection{Synthetic Data}

{\em Data}:
We made stochastic block model(SBM)~(\cite{s97}) data where
the number of nodes is $210$, the number of communities is $7$,
the connection probability within a community is $0.3$, and the connection probability between communities is $0.02$.

{\em Hyper-parameter Setting}:
For the GAE encoding, we use temperature $\tau=1$, latent dimension 6, hidden dimension 32, and 200 epochs. In MDL-guided sampling, the proposal distribution uses $\sigma =0.005$ for both mixture models. We sample $m=500$ sets, generating $n'=30$ nodes per graph. Full setup details are in Appendix \ref{exp_setup}.

{\em Evaluation Metrics}:
To comprehensively evaluate the generated graphs, we bridge the latent and raw spaces.
Let $C$ be a set of additional decoded nodes.
The {\em Negative Log Likelihood} (NLL) is computed for  the resulting graph relative to the existing one, which measures the abnormality of $C.$
The {\em Conductance} (CD) of $C$ (\cite{newman}) measures  the boundary separation of $C$.
We introduce {\em Bridge-adjusted Autonomy Score} (BAS) of $C$ as
${\rm BAS}(C)  = H(C)(1-{\rm CD}(C))/\log k ,$
where for the graph with $k$ clusters,
$H(C)$ is the entropy with regard to $p_{j}$
and $p_{j}(C)$ is the probability of $C$ being assigned to the $j$th cluster.
BAS measures the degree of
sparse-yet-broad connectivity across  clusters.
{\em Modularity Variation} (MOD) of $C$ quantifies the extent to which the original modularity~(\cite{newman}) is preserved against the addition of $C$.
The larger BAS/the larger NLL, the greater the novelty of $C$ is.
The larger CD/the smaller MOD, the greater the reliability of $C$ is.
\begin{table}[b]
\vspace*{-0.3cm}
\centering
\small
\caption{Spearman correlation for novelty and reliability metrics on SBM.}
\label{table-cor}
\begin{tabular}{l ccc ccc}
\toprule
& \multicolumn{3}{c}{\textbf{vMF}} & \multicolumn{3}{c}{\textbf{GMM}} \\
\cmidrule(lr){2-4} \cmidrule(lr){5-7}
Correlation Pair & MDL & LL & KL & MDL & LL & KL \\
\midrule
Novelty vs BAS & 0.8619 & 0.8667 & 0.9016 & 0.2815 & 0.2795 & 0.3417 \\
Novelty vs NLL & 0.9671 & 0.9663 & 0.9293 & 0.7561 & 0.7464 & 0.5407 \\
Reliability vs CD & -0.8943 & -0.8981 & 0.0292 & -0.3614 & -0.3617 & -0.1161 \\
Reliability vs MOD & 0.8500 & 0.8682 & -0.1087 & 0.4110 & 0.4189 & -0.5087 \\
\bottomrule
\end{tabular}

\setlength{\tabcolsep}{2pt}
\caption{Comparison of metrics (mean$\pm$SD) across novelty and reliability-based screening pools.}
\label{table-eval1}
\begin{tabular}{l ccc ccc ccc}
\toprule
& \multicolumn{3}{c}{\textbf{MDL Score}} & \multicolumn{3}{c}{\textbf{LL Score}} & \multicolumn{3}{c}{\textbf{KL Score}} \\
\cmidrule(lr){2-4} \cmidrule(lr){5-7} \cmidrule(lr){8-10}
\textbf{Novelty} & $\epsilon_1$ & BAS$\uparrow$ & NLL$\uparrow$ & $\epsilon_1$ & BAS$\uparrow$ & NLL$\uparrow$ & $\epsilon_1$ & BAS$\uparrow$ & NLL$\uparrow$ \\
\midrule
25\% & 4.33 & $\textbf{0.65}{\pm}0.22$ & $\textbf{14.79}{\pm}1.68$ & 4.79 & $\textbf{0.66}{\pm}0.22$ & $\textbf{14.77}{\pm}1.70$ & 0.52 & $\textbf{0.68}{\pm}0.21$ & $14.45{\pm}1.96$ \\
50\% & 3.56 & $0.56{\pm}0.23$ & $13.27{\pm}2.11$ & 4.02 & $0.56{\pm}0.23$ & $13.27{\pm}2.11$ & 0.40 & $0.57{\pm}0.22$ & $13.12{\pm}2.28$ \\
100\% & -0.17 & $0.37{\pm}0.27$ & $10.14{\pm}3.94$ & 0.23 & $0.37{\pm}0.27$ & $10.14{\pm}3.94$ & 0.01 & $0.37{\pm}0.27$ & $10.14{\pm}3.94$ \\
\midrule
\textbf{Reliability} & $\epsilon_2$ & CD$\uparrow$ & MOD$\downarrow$ & $\epsilon_2$ & CD$\uparrow$ & MOD$\downarrow$ & $\epsilon_2$ & CD$\uparrow$ & MOD$\downarrow$ \\
\midrule
25\% & 6.87 & $\textbf{0.37}{\pm}0.10$ & $\textbf{0.16}{\pm}0.05$ & 4.81 & $\textbf{0.37}{\pm}0.10$ & $\textbf{0.16}{\pm}0.05$ & 1.10 & $0.19{\pm}0.13$ & $0.23{\pm}0.04$ \\
50\% & 8.67 & $0.29{\pm}0.12$ & $0.19{\pm}0.05$ & 6.56 & $0.29{\pm}0.12$ & $0.19{\pm}0.05$ & 1.14 & $0.17{\pm}0.13$ & $0.24{\pm}0.04$ \\
100\% & 15.05 & $0.18{\pm}0.14$ & $0.23{\pm}0.05$ & 12.96 & $0.18{\pm}0.14$ & $0.23{\pm}0.05$ & 1.35 & $0.18{\pm}0.14$ & $0.23{\pm}0.05$ \\
\bottomrule
\end{tabular}

\end{table}

{\em Methods for Comparison}:
We define LL as the variant of our MDL-based method in which the code-length is replaced with the negative log-likelihood, and KL as the variant in which the difference in code-lengths is replaced with the Kullback-Leibler divergence. We compare our method with LL and KL. The comparison with LL serves as an ablation study to assess the effect of using MDL. Since \cite{icdm24} employs GMM with KL, our comparison covers this approach.

{\em Results}:
Table \ref{table-cor} shows the results on comparison of vMF mixtures with GMM in terms of Spearman correlation of evaluation metrics with novelty/reliability, for our method~(MDL), LL and KL. vMF mixtures tend to exhibit stronger correlations with the metrics than GMMs uniformly over all methods.
This also implies that our method can control novelty and reliability better than the method in \cite{icdm24}.
Therefore, our subsequent analysis focuses exclusively on the overall robust vMF formulation.

Table \ref{table-eval1} details the expansion performance under vMF, evaluating the expanded communities filtered by the top 25\%, 50\%, and 100\% of latent scores. Crucially, higher latent novelty consistently translates to higher values of BAS and NLL. Regarding reliability, strict filtering by MDL and LL successfully yields expanded structures with enhanced CD and preserved MOD.
This implies that novelty and reliability
of generated graphs can be well controlled with novelty and reliability scores in the latent space.
In contrast, KL scores do not reflect these topological improvements. This confirms that both MDL and LL effectively guide reliable expansion on synthetic data; however, our subsequent evaluation on real-world networks will reveal the distinct superiority of MDL over LL.

\subsection{Real Data}\label{real_data}

{\em Data}: We evaluated our proposed method on two benchmark datasets provided by \cite{shchur2018pitfalls}: Amazon Computers (AC, 13,752 nodes, 245,861 edges, 10 classes) and Coauthor Physics (CP, 34,493 nodes, 495,924 edges, 5 classes).

{\em Hyper-parameter Setting}:
We set the hyper-parameters for GAE as
latent dimension$=8$, hidden dimension$=32$,
and
the number of epochs $=300$.
See Appendix \ref{exp_setup} for more details.

\begin{table}[htb]
\centering
\caption{Reliability comparison results (MDL vs LL vs KL).}
\label{rel_table}
\vspace{0.2cm}
\small
\setlength{\tabcolsep}{2.5pt}
\resizebox{\linewidth}{!}{
\begin{tabular}{ll ccc @{\hspace{1em}} ccc @{\hspace{1em}} ccc}
\toprule
& & \multicolumn{3}{c}{\textbf{MDL Score}} & \multicolumn{3}{c}{\textbf{LL Score}} & \multicolumn{3}{c}{\textbf{KL Score}} \\
\cmidrule(lr){3-5} \cmidrule(lr){6-8} \cmidrule(lr){9-11}
\textbf{Data} & \textbf{Top} & $\epsilon_2$ & CD$\uparrow$ & \makecell{{MOD}$\downarrow$ \\ {\footnotesize($\times 10^{-3}$)}} & $\epsilon_2$ & CD$\uparrow$ & \makecell{{MOD}$\downarrow$ \\ {\footnotesize($\times 10^{-3}$)}} & $\epsilon_2$ & CD$\uparrow$ & \makecell{{MOD}$\downarrow$ \\ {\footnotesize($\times 10^{-3}$)}} \\
\midrule

\multirow{3}{*}{AC} &
25\% & 0.42 & \textbf{0.82} $\pm$ 0.21 & \textbf{2.3} $\pm$ 4.6 & 0.51 & 0.46 $\pm$ 0.26 & 17.0 $\pm$ 22.1 & $1.1{\times}10^{-3}$ & 0.46 $\pm$ 0.29 & 17.3 $\pm$ 21.0 \\
& 50\% & 0.69 & 0.67 $\pm$ 0.28 & 4.8 $\pm$ 7.2 & 0.59 & 0.48 $\pm$ 0.29 & 16.2 $\pm$ 26.6 & $1.3{\times}10^{-3}$ & 0.47 $\pm$ 0.31 & 18.0 $\pm$ 25.1 \\
& 100\% & 1.78 & 0.44 $\pm$ 0.32 & 20.6 $\pm$ 28.8 & 1.23 & 0.44 $\pm$ 0.32 & 20.6 $\pm$ 28.8 & $2.2{\times}10^{-3}$ & 0.44 $\pm$ 0.32 & 20.6 $\pm$ 28.8 \\
\midrule

\multirow{3}{*}{CP} &
25\% & 0.73 & \textbf{0.74} $\pm$ 0.13 & \textbf{7.3} $\pm$ 9.4 & 0.38 & \textbf{0.75} $\pm$ 0.11 & 8.9 $\pm$ 11.0 & $2.5{\times}10^{-4}$ & \textbf{0.73} $\pm$ 0.10 & 26.2 $\pm$ 30.0 \\
& 50\% & 0.89 & 0.66 $\pm$ 0.16 & 7.4 $\pm$ 9.1 & 0.49 & 0.69 $\pm$ 0.14 & 13.2 $\pm$ 20.0 & $3.1{\times}10^{-4}$ & 0.68 $\pm$ 0.14 & 16.8 $\pm$ 24.0 \\
& 100\% & 1.75 & 0.61 $\pm$ 0.17 & 15.4 $\pm$ 19.8 & 1.03 & 0.61 $\pm$ 0.17 & 15.4 $\pm$ 19.8 & $7.1{\times}10^{-4}$ & 0.61 $\pm$ 0.17 & 15.4 $\pm$ 19.8 \\
\bottomrule

\end{tabular}
}
\end{table}

{\em Results}:
Table \ref{rel_table} reports MOD and CD under reliability thresholds(25\%, 50\%, and 100\%) on the AC and CP datasets. This demonstrates MDL's superior controllability. Under stricter filtering(from 100\% down to 25\%), only MDL consistently extracts topologically reliable structures, shown by monotonically decreasing MOD and increasing CD. Conversely, LL and KL fail to achieve this.
In terms of novelty, MDL is superior to LL and KL, yet it does not always exhibit a monotonic relationship with the threshold. While excluded from Table \ref{rel_table}, we discuss this in further detail using scatter plots of metrics versus novelty/reliability in Appendix \ref{full_results}.

Most importantly, the precise control over reliability is the definitive cornerstone of our framework. By explicitly preserving the global structural integrity through the MDL-guided reliability condition, the model can safely and effectively explore the latent space to synthesize truly novel patterns. Ultimately, these results confirm that harnessing MDL's robust reliability control is what enables desirable and principled novelty generation, even within highly complex real-world networks.

\section{Conclusion}\label{conclusion}

This paper presented an information-theoretic framework for graph novelty generation. While modern generative models often appear to produce novel and reliable outputs, such properties are typically emergent and uncontrolled; in contrast, our framework explicitly formulates novelty and reliability and provides theoretical guarantees on their controllability.
Our approach models latent representations using finite mixture models and generates new samples via MDL-guided sampling, where novelty and reliability are characterized through local and global description-length criteria. Rather than guaranteeing both properties absolutely, we analyzed the probabilities of their respective misclassification events and showed their exponential convergence.
Experiments on synthetic and benchmark datasets demonstrated that the proposed method generates structurally novel yet reliable graphs, with controlled quantified risk.
This ability to explicitly regulate novelty and reliability distinguishes our framework from existing generative approaches such as diffusion models and GANs.

A limitation of this work is that the framework is demonstrated on graphs with fixed  threshold designs, leaving their optimal selection as future work.
It is also left for future study how to scale MDL-guided sampling to high dimensionality setting.
Future work also includes extending the framework to data modalities beyond graphs and exploring how
the generated samples can be interpreted.

\pagebreak

\section*{Impact Statement}\label{impact}

This paper presents work whose goal is to advance the field of Machine
Learning and AI. There are many potential societal consequences of our work, none
that we feel must be specifically highlighted here.

\clearpage

\newpage
\appendix

\section{Proof of Proposition \ref{prop1}}\label{proofprop}

\begin{proof}
According to \cite{optimal,gwd,mdlbook}, the parametric complexity can be approximated as follows: Letting $d$ be the number of parameters and $\theta$ be a real-valued parameter vector, for sufficiently large $n,$
\begin{equation}\label{asymp}
C_{n}=\int  \max_{\theta}p({\bm z}; \theta)d{\bm z}\approx  \frac{d}{2}\log \frac{n}{2\pi}+\log \int \sqrt{|I(\theta)|}d\theta ,
\end{equation}
where $I(\theta )$ is the Fisher information matrix; $I(\theta )\buildrel \rm def \over =E_{p}[-\frac{\partial ^{2}\log p(z; \theta )}{\partial \theta \partial \theta ^{\top}}]$, and $|I(\theta)|$ is the determinant of $I(\theta)$.

For the vMF-distribution, the determinant of  Fisher information matrix is given by a straightforward calculation as follows:
\begin{equation}\label{fisher}
|I(\mu , \lambda )|^{\frac{1}{2}}=\left( \frac{A_{d}(\lambda )}{\lambda} \right)^{\frac{d-1}{2}}\sqrt{ 1-\frac{(d-1)A_{d}(\lambda)}{\lambda}-A_{d}(\lambda )^{2}},
\end{equation}
where $A_{d}(\lambda )=I_{d/2}(\lambda )/I_{d/2-1}(\lambda )$.

For the evaluation of (\ref{asymp}), we prepare the following lemma.
\begin{lemma}\label{vmf-prmup}
The following inequality holds:
\begin{equation}\label{fisherm-wolog}
\int ^{\infty}_{0}\int _{||\mu||=1}|I(\mu , \lambda )|^{1/2}d\mu d\lambda \leq \frac{\pi ^{\frac{d+1}{2}}\sqrt{d}\Gamma (\frac{d-1}{2}) }{(d-1)^{\frac{d-1}{2}}\Gamma (\frac{d}{2})}\left( \frac{1}{2\Gamma (\frac{d}{2}+1)}+\frac{1}{\Gamma (\frac{d}{2})}\right).
\end{equation}
\end{lemma}
By (\ref{asymp}) and (\ref{fisherm-wolog}),
the parametric complexity of vMF distribution is upper-bounded by
\begin{eqnarray*}\label{fisherm}
\log C_{n}&= &\frac{d}{2}\log \frac{n}{2\pi}+ \log \int ^{\infty}_{0}\int _{||\mu||=1}|I(\mu , \lambda )|^{1/2}d\mu d\lambda \nonumber \\
&\leq&  \frac{d}{2}\log \frac{n}{2\pi} + \log \frac{\pi ^{\frac{d+1}{2}}\sqrt{d}\Gamma (\frac{d-1}{2}) }{(d-1)^{\frac{d-1}{2}}\Gamma (\frac{d}{2})}\left( \frac{1}{2\Gamma (\frac{d}{2}+1)}+\frac{1}{\Gamma (\frac{d}{2})}\right)\\
&=& \frac{d}{2}\log \frac{n}{2\pi}+
\log \frac{\pi ^{\frac{d+1}{2}}(d+1)\Gamma (\frac{d-1}{2}) }{\sqrt{d}(d-1)^{\frac{d-1}{2}}\Gamma (\frac{d}{2})^2}.
\end{eqnarray*}
\end{proof}

Below we give a proof of Lemma ~\ref{vmf-prmup}.
\begin{proof}
From (\ref{fisher}), we obtain
\begin{align}\label{prmup}
&\textstyle\int_0^{\infty} \int_{\|\mu\|=1} \sqrt{| I(\mu,\,\lambda)|} \mathrm d\mu \mathrm d\lambda \nonumber \\
&\textstyle= \int_0^{\infty} \int_{\|\mu\|=1} \sqrt{\left(\frac{A_d(\lambda)}{\lambda}\right)^{d-1}
     \left(1-\frac{(d-1)\,A_d(\lambda)}{\lambda}-A_d(\lambda)^2\right)} \mathrm d\mu \mathrm d\lambda \nonumber \\
&\textstyle= \frac{2\pi^{\frac d2}}{\Gamma(\frac d2)} \int_0^{\infty}\sqrt{\left(\frac{A_d(\lambda)}{\lambda}\right)^{d-1}
     \left(1-\frac{(d-1)\,A_d(\lambda)}{\lambda}-A_d(\lambda)^2\right)} \mathrm d\lambda.
\end{align}
Let $\nu = \frac d2 -1$. \cite{amos74} gives the bounds
\begin{align*}
& \hspace{-20pt} \frac{\lambda}{\nu+\frac12 +(\lambda^2 + (\nu + \frac32)^2 )^{1/2}} \leq A_d(\lambda)
 \leq \frac{\lambda}{\nu+\frac12 +(\lambda^2 + (\nu + \frac12)^2 )^{1/2}} .
\end{align*}

Then the integral in (\ref{prmup}) is upper bounded like

\begin{align}\label{prmup3}
&\textstyle \int_0^{\infty}\sqrt{\left(\frac{A_d(\lambda)}{\lambda}\right)^{d-1}
     \left\{\,1-\frac{(d-1)\,A_d(\lambda)}{\lambda}-A_d(\lambda)^2\right\}} \mathrm d\lambda \nonumber \\
&\textstyle \leq \int_0^{\infty}\sqrt{\left( \frac{1}{\nu + \frac 12 +\sqrt{\lambda^2 +(\nu +\frac12)^2}} \right)^{d-1} }\nonumber \\
&\textstyle \hspace{5pt}\sqrt{
     \left\{\,1-\frac{(d-1)}{\nu + \frac 12 +\sqrt{\lambda^2 +(\nu +\frac32)^2}}-\left(\frac{\lambda}{\nu + \frac 12 +\sqrt{\lambda^2 +(\nu +\frac32)^2}} \right)^2\right\}} \mathrm d\lambda\nonumber \\
&\textstyle= \sqrt{2\nu+2} \int_0^{\infty} \frac{\left(\nu+\frac12+\sqrt{\lambda^2 + (\nu+\frac12)^2}\right)^{-\nu-\frac12}}{\nu+\frac12+\sqrt{\lambda^2 + (\nu+\frac32)^2}} \mathrm d\lambda\nonumber \\
&\textstyle\leq \sqrt{2\nu+2} \int_0^{\infty} \left(\nu+\frac12+\sqrt{\lambda^2 + (\nu+\frac12)^2}\right)^{-\nu-\frac32} \mathrm d\lambda.
\end{align}

For brevity, set $a=\nu+\frac{1}{2}$ and consider the change of variables $\lambda = at$.
Then

\begin{align}\label{prmup2}
&\int_0^{\infty} \left(\nu+\frac12+\sqrt{\lambda^2 + (\nu+\frac12)^2}\right)^{-\nu-\frac32} \mathrm d\lambda \nonumber\\
&= \int_0^{\infty} (a+\sqrt{\lambda^2 +a^2})^{-a-1} \mathrm d\lambda \nonumber\\
&= a^{-a} \int_0^{\infty} (1+\sqrt{t^2 +1})^{-a-1} \mathrm dt.
\end{align}

Next, we apply $t=\sinh u$, followed by $v=\tanh(u/2)$ and $w=v^{2}$, which leads to

\begin{align*}
(\ref{prmup2})&= a^{-a} \int_0^{\infty} \frac{\cosh u}{\left( 1+\cosh u\right)^{a+1}}  \,\mathrm du \\
&= (2a)^{-a} \int_0^1 (1+v^2)(1-v^2)^{a-1} \mathrm dv \\
&= (2a)^{-a} \int_0^1 (1+w)(1-w)^{a-1} \cdot \frac{\mathrm dw}{2\sqrt w}  \\
&= \frac{(2a)^{-a}}{2} \left\{ B\left(\frac 32, \, a\right) + B\left(\frac 12, \, a\right) \right\} \\
&= \frac{(2a)^{-a}}{2} \left\{ \frac{\Gamma(\frac32)\Gamma(a)}{\Gamma\left(\frac 32+a\right)} +\frac{\Gamma(\frac12)\Gamma(a)}{\Gamma\left(\frac 12+a\right)}\right\} \\
&= \frac{(2a)^{-a}}{2} \sqrt{\pi} \Gamma(a) \left( \frac 1{2\Gamma\left(\frac 32+a\right)} + \frac 1{\Gamma\left(\frac 12+a\right)}\right).
\end{align*}
Here $B(\cdot,\cdot)$ denotes the beta function.
Inserting the above identity into (\ref{prmup3}) and using (\ref{prmup}) yields Lemma~\ref{vmf-prmup}.
\end{proof}
We adopt the upper bound (\ref{vmfnml}) as the parametric complexity for the vMF distribution.

The NML code-length

The NML code-length for the Gaussian radial component $r$ is computed as follows:
\begin{proposition} \cite{mdlbook}(p.23)
For an i.i.d. sequence $r^n=r_1\dots r_{n},$
the NML code-length for $r^n$ for the 1-dimensional Gaussian model is calculated as follows:
\begin{eqnarray}\label{1gauss}
{\mathcal L}_{_{\rm NML}}(r^n)=\frac{n}{2}\log (2\pi e\hat{\tau})+\log \frac{n}{2\pi}+O(1).
\end{eqnarray}
where $\hat{\tau}=\frac{1}{n}\sum _{i=1}^{n}(r_{i}-\hat{\mu})^{2}$ for $\hat{\mu}=\frac{1}{n}\sum _{i=1}^{n}r_{i}$.
\end{proposition}
We adopt (\ref{1gauss}) ignoring $O(1)$ term as the NML code-length for $r^{n}$ relative to 1-dimensional Gaussian distribution.

Totally, the NML code-length for $z^{n}$ is calculated as the sum of the NML code-length for angle component data relative to vMF distribution and that for radius component data relative to the 1-dimensional Gaussian distribution, as follows:
\begin{eqnarray}
{\mathcal L}_{_{\rm NML}}(z^{n})&=&\sum _{i=1}^{n}(-\hat{\lambda}\hat{\mu}^{\top}\phi _{i} )+n\log C(\hat{\lambda})+\frac{d}{2}\log \frac{n}{2\pi}+
\log \frac{\pi ^{\frac{d+1}{2}}(d+1)\Gamma (\frac{d-1}{2}) }{\sqrt{d}(d-1)^{\frac{d-1}{2}}\Gamma (\frac{d}{2})^2}\nonumber \\
& &+
\frac{n}{2}\log (2\pi e\hat{\tau})+\log \frac{n}{2\pi},
\end{eqnarray}
where $\hat{\mu}$ and $\hat{\lambda}$ are MLE of $\mu$ and $\lambda$.

\section{Gaussian Mixture and Its NML Code-length}\label{gaussiannml}

We briefly introduce a {\em Gaussian mixture model} (GMM),
whose probability density function over ${\mathbb R}^{d}$ is given by
\begin{eqnarray*}
p(z; \mu , \Sigma)=\frac{1}{(2\pi)^{\frac{d}{2}}|\Sigma |^{\frac{1}{2}}}\exp \left(-\frac{1}{2}(z-\mu)^{\top}\Sigma ^{-1}
(z-\mu) \right),
\end{eqnarray*}
where $\mu \in {\mathbb R}^{d}$ and $\Sigma \in {\mathbb R}^{d\times d}$.

Introducing a latent variable $w$ which indicates the cluster index  corresponding to $z$, for a positive integer $k$~(mixture size), we consider a GMM over the latent space, for which the complete variable model is given by:
\begin{align*}
&p(z,w;\theta )=\prod _{j=1}^{k}(\pi _{j}p(z;\mu_{j}, \Sigma _{j}))^{\delta (w=j)},
 \end{align*}
where $\pi_{j}\geq 0$, $\sum _{j=1}^{k}\pi _{j}=1$, and $\theta =(\{\pi _{j}, \mu _{j}, \Sigma _{j}\})$.
MLE of the parameter $\theta$ and latent variable $w$ can be obtained
with the EM algorithm~(\cite{nh}).

Below we give a formula on the NML code-length for a Gaussian distribution.
\begin{proposition} \cite{hy19}
The NML code-length for an i.i.d. sequence ${\bm z}$
relative to a multivariate Gaussian distribution is upper-bounded as:
\begin{align}\label{gmm-nml}
& {\mathcal L}_{_{\rm NML}}({\bm z})\leq \frac{dn}{2}\log (2\pi )+\frac{n}{2}\log |\hat{\Sigma} |+\frac{dn}{2}+\frac{dn}{2}\log \frac{n}{2e} -\frac{d(d-1)}{4}\log \pi -\sum ^{d}_{j=1}\log \Gamma \left( \frac{n-j}{2}\right)\nonumber \\
&-d\log \frac{d}{2}  -\log \Gamma \left(\frac{d}{2}+1 \right)+\frac{d}{2}\log ||\hat{\mu}||^{2}-\frac{d}{2}\sum ^{d}_{j=1}\log \lambda ^{(j)}(\hat{\Sigma}) +(d+1)\log \frac{d}{2}+\log \log \frac{R_{2}}{R_{1}}\nonumber \\
&+d\log \log \frac{\lambda _{2}}{\lambda _{1}},
\end{align}
where $d$ is the dimension of ${\mathcal Z}$, $n$ is the data length, $\hat{\mu}$ and $\hat{\Sigma}$ are MLEs of $\mu$ and $\Sigma$, $R_{1},R_{2},\lambda_{1},\lambda _{2}$ are constants such that
$R_{1}\leq ||\mu ||^{2}\leq R_{2}$ and $\forall j$, $ \lambda _{1}\leq \lambda ^{(j)}(\hat{\Sigma})\leq \lambda _{2}$ as the $j$th largest
 eigenvalue of $\hat{\Sigma}$.
\end{proposition}

We adopt the upper bound (\ref{gmm-nml}) as the NML code-length for the multivariate Gaussian distribution.
Note that the parametric complexity part does not depend on the data.

\section{Proof of Theorem \ref{thm-novelty}}\label{proofth1}
\begin{proof}
Let $n_{i}=|{\bm z}_{i}|$ and $n'=|{\bm z}'|$.
Suppose that the novelty condition (\ref{noveltycond}) is fulfilled, i.e,
for each $i$,
\begin{equation}\label{i}
{\mathcal L}_{_{\rm NML}}({\bm z}_{i}\oplus {\bm z}')-
\{{\mathcal L}_{_{\rm NML}}({\bm z}_{i} )+{\mathcal L}_{_{\rm NML}}( {\bm z}' )\}> (n_{i}+n')\epsilon _{1}.
\end{equation}
Then the error probability that for some $i$, ${\bm z}$ and ${\bm z}_{i}$ are identically distributed while (\ref{noveltycond}) is fulfilled is given as follows:
\begin{eqnarray}\label{ii}
& &{\rm Prob}[\exists i,\ {\bm z}_{i}\ {\rm and}\ {\bm z}'\ {\rm i.i.d.}\ {\rm while}\  (\ref{i})\ {\rm holds}] \nonumber \\
& &={\rm Prob}[\cup _{i=1}^{k}\{
 {\bm z}_{i}\ {\rm and}\ {\bm z}'\ {\rm i.i.d.}\ {\rm while}\  (\ref{i})\ {\rm holds}\}] \nonumber \\
& &\leq \sum _{i=1}^{k}{\rm Prob}[
 {\bm z}_{i}\ {\rm and}\ {\bm z}'\ {\rm i.i.d.}\ {\rm while}\  (\ref{i})\ {\rm holds}].
\end{eqnarray}
Under (\ref{i}), supposing that ${\bm z}\oplus {\bm z}'$ is identically distributed according to $p_{\theta _{i}}$,
\begin{eqnarray*}\label{m}
& &-\log p_{\theta_{i}}({\bm z}_{i}\oplus{\bm z}')+\log C_{n_{i}+n'}\nonumber \\
& &\geq \min _{\theta}\{-\log p_{\theta}({\bm z}_{i}\oplus{\bm z}')\}+\log C_{n_{i}+n'}\nonumber \\
& &= {\mathcal L}_{_{\rm NML}}({\bm z}_{i}\oplus {\bm z}')\\
& &\geq \{ {\mathcal L}_{_{\rm NML}}({\bm z}_{i} )+{\mathcal L}_{_{\rm NML}}( {\bm z}' )\}+(n_{i}+n')\epsilon _{1}.
\end{eqnarray*}
Therefore, we have
\begin{equation*}
 p_{\theta_{i}}({\bm z}_{i}\oplus {\bm z}')
\leq e^{-{\mathcal L}_{_{\rm NML}}({\bm z}_{i})}e^{-{\mathcal L}_{_{\rm NML}}({\bm z}')}
e^{-(n_{i}+n')\epsilon _{1}+\log C_{n_{i}+n'}}. \nonumber
\end{equation*}

Combining this with (\ref{ii}) gives
\begin{align}\label{mmm}
& {\rm Prob}[{\bm z}_{i}\ {\rm and}\ {\bm z}'\ {\rm i.i.d.}\ \ {\rm while}\  (\ref{i})\ {\rm holds}]\nonumber \\
& = \iint^{*} d{\bm z}_{i}d{\bm z}'p_{\theta _{i}}({\bm z_{i}}\oplus {\bm z}')\nonumber \\
 & \leq \iint ^{*}
  e^{-{\mathcal L}_{_{\rm NML}}({\bm z}_{i})}e^{-{\mathcal L}_{_{\rm NML}}({\bm z}')}
  e^{-(n_{i}+n')\epsilon _{1}+\log C_{n_{i}+n'}}\nonumber \\
& \leq \int d{{\bm z}_{i}} e^{-{\mathcal L}_{_{\rm NML}}({\bm z}_{i})}\int d{{\bm z}'}e^{-{\mathcal L}_{_{\rm NML}}({\bm z}')}e^{-(n_{i}+n')\epsilon _{1}+\log C_{n_{i}+n'}}\nonumber \\
& =\exp(-(n_{i}+n')\epsilon _{1}+\log C_{n_{i}+n'}),
\end{align}
where the integral $\iint ^{*}$ is taken over the domain such that ${\bm z}_{i}\oplus {\bm z}'$ satisfying (\ref{i}).

Note that the NML code-length is a prefix code-length, and hence the Kraft inequality holds; $\int d{{\bm z}}e^{-{\mathcal L}_{_{\rm NML}}({\bm z})}\leq 1$. This holds even if the parametric complexity (\ref{nml}) is replaced with its upper-bound.
If the penalty-term is less than the parametric complexity, the Kraft inequality does not hold. Hence the usage of the NML code-length is essential in this proof.

 Plugging (\ref{mmm}) into (\ref{ii}) yields
 \begin{align*}
& {\rm Prob}[\exists i,\ {\bm z}_{i}\ {\rm and}\ {\bm z}'\ {\rm i.i.d.}\ {\rm while}\  (\ref{i})\ {\rm holds} ]\\
&\leq \sum _{i=1}^{k
}\exp(-(n_{i}+n')\epsilon _{1}+\log C_{n_{i}+n'}).
 \end{align*}
 This completes the proof.
 \end{proof}

\section{Proof of Theorem \ref{thm-reliability}}\label{proofth2}
\begin{proof}
Suppose that $({\bm z}\oplus{\bm z'}, {\bm w}\oplus {\bm w}')$ is distributed according to the concatenated distribution $p_{\theta}({\bm z}, {\bm w})p_{\theta'}({\bm z}', {\bm w}')$ where $\theta \neq \theta'$.
Let $\bar{C_{n}}=\sup _{{\bm w}}C_{n}$ be an upper bound on the parametric complexity for the FMM where the supremum is achieved by $n_{i}=n/k$ for each $i$.
When (\ref{reliabilitycond}) is fulfilled,  the following inequality holds:
\begin{align*}
&{\mathcal L}_{_{\rm DNML}}({\bm z}\oplus {\bm z}', {\bm w}\oplus {\bm w}')\\
&\leq {\mathcal L}_{_{\rm DNML}}({\bm z}, {\bm w})+
{\mathcal L}_{_{\rm DNML}}({\bm z}', {\bm w}')+2(n+n')\epsilon _{2}\\
& =\min _{\theta}(-\log p_{\theta}({\bm z}, {\bm w}))+\log \bar{C}_{n}+\min _{\theta'}(-\log p_{\theta'}({\bm z}', {\bm w}'))+\log C_{n'}+2(n+n')\epsilon _{2}\\
& \leq -\log p_{\theta}({\bm z}, {\bm w})
p_{\theta'}({\bm z}', {\bm w}')+\log \bar{C}_{n}C_{n'}+2(n+n')\epsilon _{2}.
\end{align*}
This yields
\begin{eqnarray}\label{mu}
1&\leq& \left(\frac{e^{-{\mathcal L}_{\rm DNML}({\bm z}\oplus {\bm z}', {\bm w}\oplus {\bm w}')}}{p_{\theta}({\bm z}, {\bm w})
p_{\theta'}({\bm z}', {\bm w}')}\right)^{\frac{1}{2}}\times \exp \left( \frac{1}{2}\log \bar{C}_{n}C_{n'}+(n+n')\epsilon _{2}\right).
\end{eqnarray}

For $\theta \neq \theta'$, we employ  (\ref{mu}) to obtain
\begin{align*}
& {\rm Prob}[\ ({\bm z},{\bm w})\sim p_{\theta}\ {\rm and}\ ({\bm z}', {\bm w}')\sim p_{\theta'}\  {\rm while\ (\ref{reliabilitycond})\ holds}]\nonumber \\
& =\iiiint^{*} d{\bm z}d{\bm w}d{\bm z}'d
{\bm w}'
p_{\theta}({\bm z}, {\bm w})p_{\theta'}({\bm z}', {\bm w}')\nonumber \\
& \leq \iiiint^{*} d{\bm z}d{\bm w}d{\bm z}'d
{\bm w}'
p_{\theta}({\bm z}, {\bm w})p_{\theta'}({\bm z}', {\bm w}')\\
&\ \ \ \ \ \times \left(\frac{e^{-{\mathcal L}_{\rm DNML}}({\bm z}\oplus {\bm z}', {\bm w}\oplus {\bm w}')}{p_{\theta}({\bm z}, {\bm w})
p_{\theta'}({\bm z}', {\bm w}')}\right)^{\frac{1}{2}} e^{ \frac{1}{2}\log \bar{C}_{n}C_{n'}+(n+n')\epsilon _{2}}\nonumber \\
& \leq \iiiint^{*} d{\bm z}d{\bm w}d{\bm z}'d
{\bm w}'
\left(e^{-{\mathcal L}_{\rm DNML}({\bm z}\oplus {\bm z}', {\bm w}\oplus {\bm w}')}p_{\theta}({\bm z},{\bm w})p_{\theta'}({\bm z}', {\bm w}')\right)^{\frac{1}{2}}\times e^{ \frac{1}{2}\log \bar{C}_{n}C_{n'}+(n+n')\epsilon _{2}}\\
& =\exp \left[-2(n+n')d(p_{_{\rm DNML}}, p_{\theta}\otimes p_{\theta' })+\frac{1}{2}\log \bar{C}_{n}C_{n'}+(n+n')\epsilon _{2}\right],
\end{align*}
where the integral $\iiiint ^{*}$ is taken over the domain such that
$({\bm z},{\bm w}) ({\bm z}',{\bm w}')$ satisfying (\ref{reliabilitycond}).
We used the property that $e^{-{\mathcal L}_{_{\rm DNML}}({\bm z}\oplus {\bm z}', {\bm w}\oplus {\bm w}')}$ forms a probability distribution over ${\mathcal Z}^{n+n'}\times {\mathcal W}^{n+n'}$, because the  DNML code-length is the prefix code-length, and hence satisfies the Kraft inequality with equality. The bound does not hold for any other loss functions with less penalty terms than the DNML.
This completes the proof.
\end{proof}

\section{Comprehensive Results and Scatter Plot Analysis for Novelty Evaluation}\label{full_results}
This section supplements the results presented in Section~\ref{real_data}, which discussed model controllability on the real-world datasets: Amazon Computers (AC) and Coauthor Physics (CP).
The comprehensive results are provided in Table~\ref{full_table}. As observed in the experiments with the SBM datasets, the downstream indices reflect the novelty scores across all metrics. However, MDL and LL may appear to underperform compared to KL. Below, we clarify that this apparent underperformance does not imply that MDL and LL are fundamentally inferior.

\begin{table}[bht]
\centering
\caption{Comprehensive evaluation report (MDL vs LL vs KL).}
\label{full_table}
\vspace{0.2cm}
\small
\setlength{\tabcolsep}{2.5pt}
\resizebox{\linewidth}{!}{
\begin{tabular}{ll ccc @{\hspace{1em}} ccc @{\hspace{1em}} ccc}
\toprule
& & \multicolumn{3}{c}{\textbf{MDL Score}} & \multicolumn{3}{c}{\textbf{LL Score}} & \multicolumn{3}{c}{\textbf{KL Score}} \\
\cmidrule(lr){3-5} \cmidrule(lr){6-8} \cmidrule(lr){9-11}

\multicolumn{11}{c}{\textbf{Novelty Metrics}} \\
\midrule
\textbf{Data} & \textbf{Top} & $\epsilon_1$ & BAS$\uparrow$ & NLL$\uparrow$ & $\epsilon_1$ & BAS$\uparrow$ & NLL$\uparrow$ & $\epsilon_1$ & BAS$\uparrow$ & NLL$\uparrow$ \\
\midrule

\multirow{3}{*}{AC}
& 25\% & 0.33 & 0.33 $\pm$ 0.22 & 6.84 $\pm$ 0.74 & 0.33 & 0.33 $\pm$ 0.22 & 6.84 $\pm$ 0.74 & $7.6{\times}10^{-4}$ & 0.39 $\pm$ 0.22 & 6.67 $\pm$ 0.77 \\
& 50\% & 0.27 & 0.33 $\pm$ 0.22 & 6.49 $\pm$ 0.78 & 0.27 & 0.33 $\pm$ 0.22 & 6.49 $\pm$ 0.78 & $6.0{\times}10^{-4}$ & 0.35 $\pm$ 0.22 & 6.45 $\pm$ 0.80 \\
& 100\% & 0.09 & 0.32 $\pm$ 0.21 & 5.91 $\pm$ 0.98 & 0.09 & 0.32 $\pm$ 0.21 & 5.91 $\pm$ 0.98 & $1.4{\times}10^{-4}$ & 0.32 $\pm$ 0.21 & 5.91 $\pm$ 0.98 \\
\midrule

\multirow{3}{*}{CP}
& 25\% & 0.28 & 0.15 $\pm$ 0.15 & 4.14 $\pm$ 0.54 & 0.28 & 0.15 $\pm$ 0.15 & 4.14 $\pm$ 0.54 & $1.2{\times}10^{-4}$ & 0.20 $\pm$ 0.15 & 4.44 $\pm$ 0.45 \\
& 50\% & 0.19 & 0.17 $\pm$ 0.15 & 4.14 $\pm$ 0.50 & 0.19 & 0.17 $\pm$ 0.15 & 4.14 $\pm$ 0.50 & $7.1{\times}10^{-5}$ & 0.19 $\pm$ 0.14 & 4.19 $\pm$ 0.46 \\
& 100\% & $5.2{\times}10^{-4}$ & 0.13 $\pm$ 0.13 & 3.58 $\pm$ 0.81 & $1.5{\times}10^{-3}$ & 0.13 $\pm$ 0.13 & 3.58 $\pm$ 0.81 & $8.5{\times}10^{-7}$ & 0.13 $\pm$ 0.13 & 3.58 $\pm$ 0.81 \\
\midrule\midrule

\multicolumn{11}{c}{\textbf{Reliability Metrics}} \\
\midrule
\textbf{Data} & \textbf{Top} & $\epsilon_2$ & CD$\uparrow$ & \makecell{{MOD}$\downarrow$ \\ {\footnotesize($\times 10^{-3}$)}} & $\epsilon_2$ & CD$\uparrow$ & \makecell{{MOD}$\downarrow$ \\ {\footnotesize($\times 10^{-3}$)}} & $\epsilon_2$ & CD$\uparrow$ & \makecell{{MOD}$\downarrow$ \\ {\footnotesize($\times 10^{-3}$)}} \\
\midrule

\multirow{3}{*}{AC}
& 25\% & 0.42 & \textbf{0.82} $\pm$ 0.21 & \textbf{2.3} $\pm$ 4.6 & 0.51 & 0.46 $\pm$ 0.26 & 17.0 $\pm$ 22.1 & $1.1{\times}10^{-3}$ & 0.46 $\pm$ 0.29 & 17.3 $\pm$ 21.0 \\
& 50\% & 0.69 & 0.67 $\pm$ 0.28 & 4.8 $\pm$ 7.2 & 0.59 & 0.48 $\pm$ 0.29 & 16.2 $\pm$ 26.6 & $1.3{\times}10^{-3}$ & 0.47 $\pm$ 0.31 & 18.0 $\pm$ 25.1 \\
& 100\% & 1.78 & 0.44 $\pm$ 0.32 & 20.6 $\pm$ 28.8 & 1.23 & 0.44 $\pm$ 0.32 & 20.6 $\pm$ 28.8 & $2.2{\times}10^{-3}$ & 0.44 $\pm$ 0.32 & 20.6 $\pm$ 28.8 \\
\midrule

\multirow{3}{*}{CP}
& 25\% & 0.73 & \textbf{0.74} $\pm$ 0.13 & \textbf{7.3} $\pm$ 9.4 & 0.38 & \textbf{0.75} $\pm$ 0.11 & 8.9 $\pm$ 11.0 & $2.5{\times}10^{-4}$ & \textbf{0.73} $\pm$ 0.10 & 26.2 $\pm$ 30.0 \\
& 50\% & 0.89 & 0.66 $\pm$ 0.16 & 7.4 $\pm$ 9.1 & 0.49 & 0.69 $\pm$ 0.14 & 13.2 $\pm$ 20.0 & $3.1{\times}10^{-4}$ & 0.68 $\pm$ 0.14 & 16.8 $\pm$ 24.0 \\
& 100\% & 1.75 & 0.61 $\pm$ 0.17 & 15.4 $\pm$ 19.8 & 1.03 & 0.61 $\pm$ 0.17 & 15.4 $\pm$ 19.8 & $7.1{\times}10^{-4}$ & 0.61 $\pm$ 0.17 & 15.4 $\pm$ 19.8 \\
\bottomrule

\end{tabular}
}
\end{table}

Focusing on the Amazon Computers (AC) dataset, Figure~\ref{scatter_amazon} reveals a shift in behavior compared to the earlier SBM experiments. For the BAS metric, while KL marginally separates values along the vertical Novelty axis, MDL homogenizes them vertically. Instead, MDL predominantly separates the values along the horizontal reliability axis. Notably, this horizontal separation is highly distinct, delineating BAS even more clearly than KL does vertically. This suggests that in complex graphs, a lower global explanatory power inherently captures a form of novelty. Consequently, MDL's apparent underperformance in Table~\ref{full_table} is merely an artifact of this shifted mechanism, rather than a strict inferiority to KL.

\begin{figure}[hb]
\centering
  \vskip -0.1in
  \begin{center}
    \centerline{\includegraphics[width=0.95\textwidth]
    {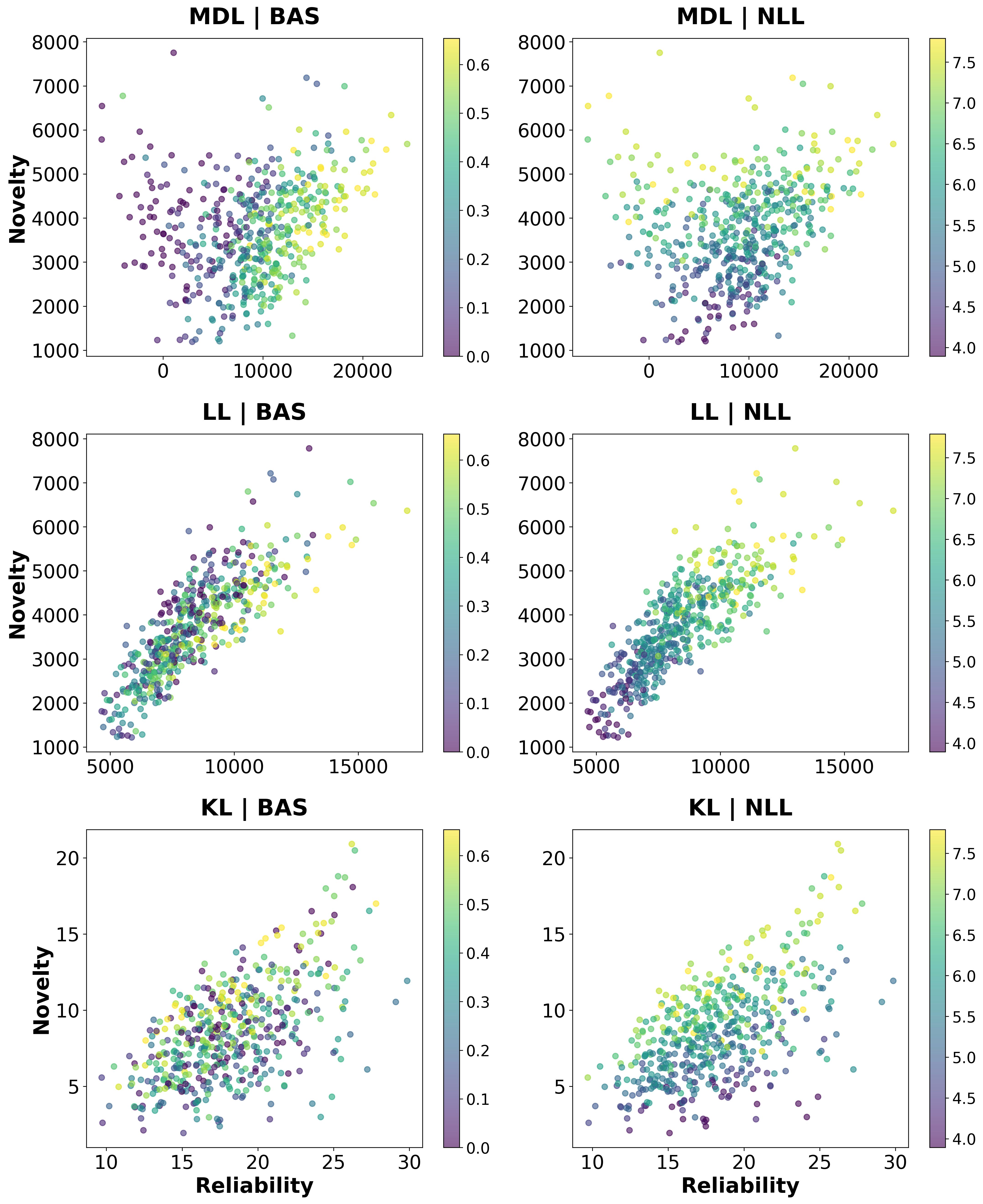}
   }
    \caption{Novelty vs Reliability on Amazon Computers (K=10). Scatter plots for the 8D latent space with EM-based structural labels. The y-axis and x-axis denote Novelty and Reliability, respectively. Panels compare statistical scorers (MDL, LL, KL) against validation metrics: BAS (left column) and NLL (right column).
        }
    \label{scatter_amazon}
  \end{center}
    \vskip -0.3in
\end{figure}

Turning to the Coauthor Physics (CP) dataset, Figure~\ref{scatter_cophys} illustrates a related phenomenon regarding the apparent underperformance of MDL observed in Table~\ref{full_table}. In complex real-world data, the relationship between the latent axes and the downstream indices can form highly non-linear boundaries. As seen in the figure, while the BAS values remain somewhat entangled in the distribution of LL, MDL segregates them across a distinct, non-linear boundary. This demonstrates that MDL achieves a significantly higher separation capability compared to both LL and KL. Consequently, the lower scores of MDL in Table~\ref{full_table} do not imply a lack of expressiveness. Rather, they suggest that standard evaluation metrics may not fully capture such complex, non-linear separation.

\begin{figure*}[hb]
  \vskip -0.1in
  \begin{center}
    \centerline{\includegraphics[width=0.95\textwidth]
    {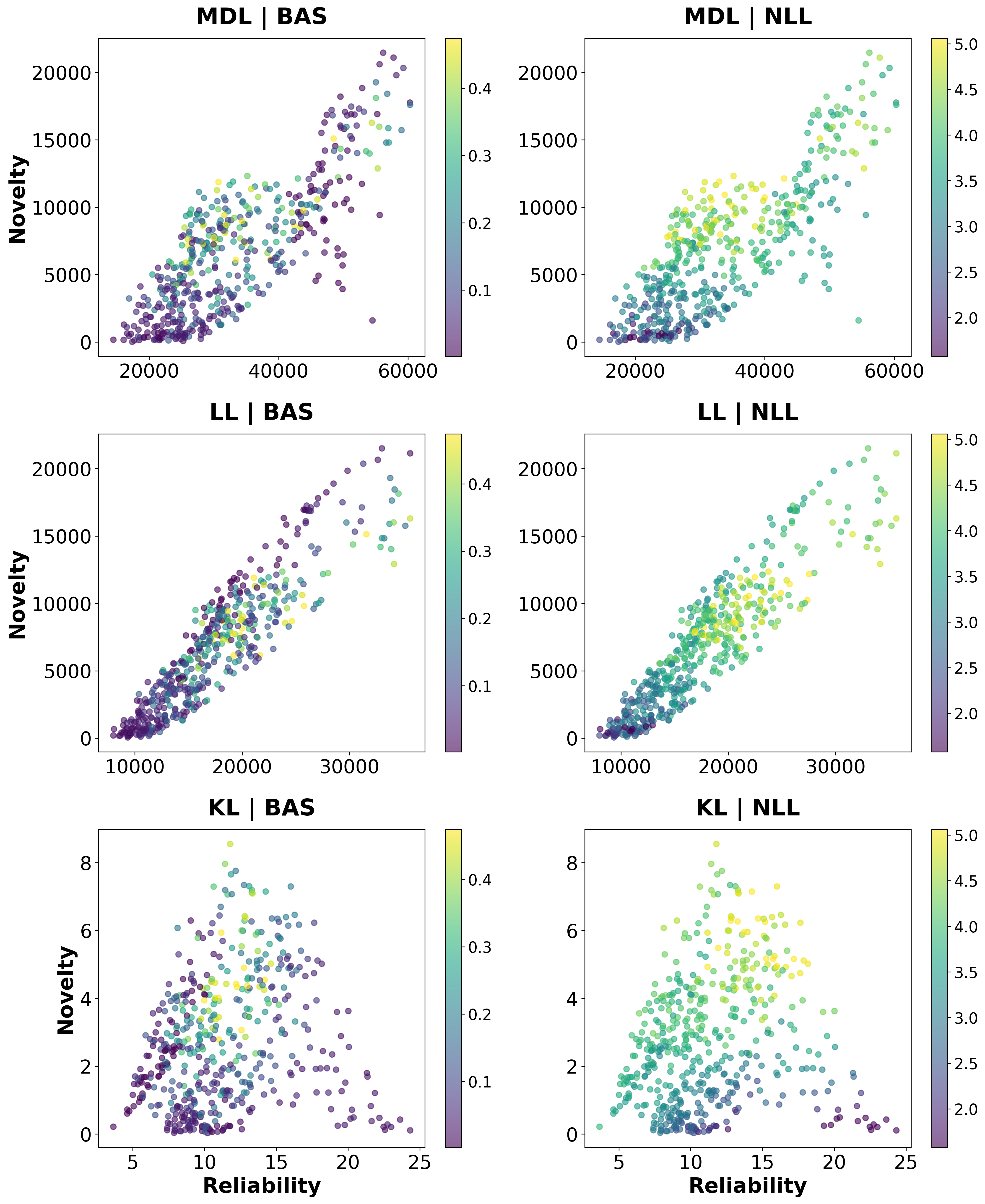}
   }
    \caption{Novelty vs. Reliability on Coauthor Physics (K=5). Similar to Figure~\ref{scatter_amazon}, but the scatter plots are visualized using the provided ground-truth class labels instead of EM-based structural labels.
        }
    \label{scatter_cophys}
  \end{center}
\end{figure*}

\section{Experimental Setup and Implementation Details}\label{exp_setup}

\subsection{Hardware and Computing Environments}
For reproducibility, all experiments throughout this paper, including those on synthetic datasets, were conducted locally on an Apple M1 MacBook Pro with 16GB of RAM. The code is written in Python using PyTorch and PyTorch Geometric.

\subsection{Datasets and Preprocessing}
In our experiments, we evaluated our graph outpainting framework on three datasets. The edge splits for link prediction were performed to validate the generation of expanded communities. Table~\ref{tab:datasets} summarizes the dataset statistics.
\begin{table}[h]
\centering
\caption{Summary of dataset statistics and edge split configurations.}
\label{tab:datasets}
\small
\begin{tabular}{lrrrr p{5.5cm}}
\toprule
\textbf{Dataset} & \textbf{Nodes} & \textbf{Edges} & \textbf{Features} & \textbf{Clusters} & \textbf{Cluster Label Source / Edge Split Method} \\
\midrule
SBM & $210$ & $\sim 2,604$ & $7$ &  $7$& Ground truth / None (Full Graph) \\
CP & $34,493$ & $247,962$ & $8,415$ & $5$ & Ground Truth (Research Field) / \newline \texttt{train\_test\_split\_edges} (90/10) \\
AC & $13,752$ & $245,861$ & $767$ & $10$ & EM Clustering on Latent Space / \newline \texttt{RandomLinkSplit} (90/10) \\
\bottomrule
\end{tabular}
\end{table}
\paragraph{Cluster Labeling for Amazon Computers:}
Since product categories do not necessarily reflect structural communities on the graph, we performed Expectation-Maximization (EM) clustering using a Gaussian Mixture Model ($K=10$) on the latent representations obtained from the Graph Autoencoder (GAE). We utilized the resulting hard-assignment labels to ensure the labels accurately reflect the structural properties of the expanded communities.

\subsection{Graph Autoencoder (GAE) Architecture}
Across all experiments, we adopted a two-layer Graph Convolutional Network (GCN) for the encoder and an Inner Product decoder. Table~\ref{tab:hyperparameters} details the hyperparameters used for each dataset.

\begin{table}[h]
\centering
\caption{GAE architecture and training hyperparameters.}
\label{tab:hyperparameters}
\small
\renewcommand{\arraystretch}{1.1}
\begin{tabular}{lccc}
\toprule
\textbf{Parameter} & \textbf{SBM} & \textbf{\shortstack{CP}} & \textbf{\shortstack{AC}} \\
\midrule
Input Dimension & $210$ & $8,415$ & $767$ \\
Hidden Dimension & $32$ & $128$ & $128$ \\
Latent Dimension ($d$) & $6$ & $8$ & $8$ \\
Decoder Type & Inner Product & Inner Product & Inner Product \\
Optimizer & Adam & Adam & Adam \\
Learning Rate & $0.01$ & $0.01$ & $0.001$ \\
Epochs & $200$ & $300$ & $300$ \\
\bottomrule
\end{tabular}
\end{table}

\subsection{Novelty Generation Process (Convex Combination)}
After estimating the distribution parameters for each cluster $k \in \{1, \dots, K\}$, we sample random Dirichlet weights $q \sim \text{Dir}(\mathbf{1}_K)$ and construct the parameters of the novel distribution via convex combination (500 iterations).

\subsubsection{GMM Configuration (SBM Only)}
For the SBM dataset under the GMM assumption, the parameters are formulated as:
\begin{align}
    \mu_{\text{new}} &= \sum_{k} q_k \mu_k + \epsilon \\
    \Sigma_{\text{new}} &= \sum_{k} q_k \Sigma_k
\end{align}

\subsubsection{vMF-Radial Configuration (All Datasets)}
To generate a new set of nodes (size $N_{\text{new}} = \sum q_k N_k$), the vMF-Radial parameters are constructed as follows:
\begin{itemize}
    \item \textbf{Direction ($\mu_{\text{dir}}$):} We compute the base direction $\mu_{\text{dir, base}} = \sum_{k} q_k \mu_{\text{dir}, k}$, add a small Gaussian noise $\epsilon \sim \mathcal{N}(0, \sigma_{\text{dir}}^2 I)$, and apply L2 normalization.
    \item \textbf{Concentration ($\kappa$):} $\kappa_{\text{new}} = \sum_{k} q_k \kappa_k$.
    \item \textbf{Radial Mean ($\mu_r$):} $\mu_{r, \text{new}} = \sum_{k} q_k \mu_{r, k}$. For Amazon Computers, an additional Gaussian noise $\mathcal{N}(0, \sigma_{\text{mean}}^2)$ is injected.
    \item \textbf{Radial Variance ($\sigma_r^2$):} $\sigma_{r, \text{new}} = \sqrt{\sum_{k} q_k \sigma_{r, k}^2}$.
\end{itemize}
When mapping the vMF-Gaussian mixtures back to Euclidean space for probability density evaluations, we apply a strict Jacobian correction following a $1/r^{d-1}$ penalty to ensure theoretical correctness in the density estimation.

Table~\ref{tab:noise} specifies the standard deviations of the additive noise.

\begin{table}[h]
\centering
\caption{Noise hyperparameters for the generative processes.}
\label{tab:noise}
\begin{tabular}{lccc}
\toprule
\textbf{Noise Parameter} & \textbf{SBM} & \textbf{CP} & \textbf{AC} \\
\midrule
Direction Noise ($\sigma_{\text{dir}}$) & $0.005$ & $0.005$ & $0.5$ \\
Radial Mean Noise ($\sigma_{\text{mean}}$) & - & - & $0.5$ \\
\bottomrule
\end{tabular}
\end{table}

\subsection{Graph Decoding and Thresholding}
To reconstruct the edges from the generated expanded communities, we calculate the inner product between the generated latent variables $W$ and the original latent variables $Z$. The threshold is adaptively determined to maintain the target density of the original graph:
\begin{itemize}
    \item \textbf{SBM \& Coauthor Physics:} Probability-based. The inner product scores are passed through a sigmoid function $\sigma(\cdot)$, and edges are formed where the score exceeds the threshold $\tau_{\text{prob}}$.
    \item \textbf{Amazon Computers:} Logit-based. The raw inner product scores are directly used, and edges are formed where the score exceeds $\tau_{\text{logit}}$.
\end{itemize}

\end{document}